\documentclass[11pt]{article}

\pdfoutput=1

\usepackage[round]{natbib}

\usepackage[T1]{fontenc}
\usepackage[latin9]{inputenc}
\usepackage{float}
\usepackage{fullpage,times,color}
\usepackage{boxedminipage}

\usepackage{amsmath,amsthm,amssymb}

\usepackage{graphicx}
\usepackage{ltxtable} 
\newcolumntype{L}{>{\centering\arraybackslash} m{0.04\columnwidth}} 
\newcolumntype{R}{>{\centering\arraybackslash} m{0.48\columnwidth}} 
\newcolumntype{S}{>{\centering\arraybackslash} m{0.32\columnwidth}} 

\newcommand{\Ocal}{\mathcal{O}}
\newcommand{\Wcal}{\mathcal{W}}
\newcommand{\norm}[1]{\|#1\|}
\newcommand{\inner}[1]{\langle#1\rangle}
\newcommand{\Ncal}{\mathcal{N}}

\newtheorem{lemma}{Lemma}

\newtheorem{theorem}{Theorem}

\DeclareMathOperator*{\E}{\mathbb{E}}

\newcommand{\reals}{\mathbb{R}}

\newcommand{\figref}[1]{Figure~\ref{#1}}

\newcommand{\thmref}[1]{Theorem~\ref{#1}}

\newcommand{\lemref}[1]{Lemma~\ref{#1}}

\newcommand{\removed}[1]{}

\newcommand{\werm}{\hat{w}}
\newcommand{\wopt}{w^*}
\newcommand{\wavg}{\bar{w}}
\newcommand{\Dcal}{\mathcal{D}}

\newenvironment{myalgo}[1]%
{
\begin{center}
\begin{boxedminipage}{0.9\linewidth}
\begin{center}
\textbf{\texttt{#1}}
\end{center}
\rm
\begin{tabbing}
....\=...\=...\=...\=...\=  \+ \kill
} %
{\end{tabbing}
\end{boxedminipage} \end{center} 
}

\title{Communication Efficient Distributed Optimization using an Approximate Newton-type Method}
\author{Ohad Shamir\\Weizmann Institute\\ of Science
\and Nathan Srebro\\Toyota Technological Institute\\ and the Technion \and
Tong Zhang\\ Rutgers University\\ and Baidu }
\date{}

\begin{document}

\maketitle

\begin{abstract}
We present a novel Newton-type method for distributed optimization, which
is particularly well suited for stochastic optimization and learning
problems.  For quadratic objectives, the method enjoys a linear rate of
convergence which provably \emph{improves} with the data size, requiring an
essentially constant number of iterations under reasonable assumptions.  We
provide theoretical and empirical evidence of the advantages of our method
compared to other approaches, such as one-shot parameter averaging and
ADMM.
\end{abstract}

\section{Introduction}

We consider the problem of distributed optimization, where each of $m$
machines has access to a function $\phi_i:\reals^d\rightarrow\reals$,
$i=1,\ldots,m$, and we would like to minimize their average
\begin{equation}
  \label{eq:phi}
  \phi(w) = \tfrac{1}{m} \sum_{i=1}^m \phi_i(w).
\end{equation}
We are particularly interested in a stochastic optimization (learning)
setting, where the ultimate goal is to minimize some stochastic
(population) objective (e.g.~expected loss or generalization error)
\begin{equation}
  \label{eq:F}
  F(w) = \E_{z\sim\Dcal}[f(w,z)]
\end{equation}
and each of the $m$ machines has access to $n$ i.i.d.~samples
$z_i^1,\ldots,z_i^n$ from the source distribution $\Dcal$, for a total
of $N=nm$ independent samples evenly and randomly distributed among
the machines.  Each machine $i$ can construct a local empirical
(sample) estimate of $F(w)$:
\begin{equation}
  \label{eq:hatFi}
  \phi_i(w) = \hat{F}_i(w) = \frac{1}{n} \sum_{j=1}^n f(w,z_i^j)
\end{equation}
and the overall empirical objective is then:
\begin{equation}
  \label{eq:hatF}
  \phi(w) = \hat{F}(w) = \frac{1}{m} \sum_{i=1}^m \hat{F}_i(w) = \frac{1}{nm}
  \sum_{i,j} f(w,z_i^j).
\end{equation}
We can then use the {\em empirical risk minimizer} (ERM)
\begin{equation}
  \label{eq:what}
  \werm = \arg\min \hat{F}(w)
\end{equation}
as an approximate minimizer of $F(w)$.  Since our interest lies mostly
with this stochastic optimization setting, we will denote
$\werm=\arg\min \phi(w)$ even when the optimization objective
$\phi(w)$ is not an empirical approximation to a stochastic objective.

When considering distributed optimization, two resources are at play: the
amount of processing on each machine, and the communication between machines.
In this paper, we focus on algorithms which alternate between a local
optimization procedure at each machine, and a communication round involving
simple map-reduce operations such as distributed averaging of vectors in
$\reals^d$. Since the cost of communication is very high in practice
\citep{bekkerman2011scaling}, our goal is to develop methods which quickly
optimize the empirical objective $\hat{F}(\cdot)$, using a minimal number of
such iterations.

\paragraph{One-Shot Averaging} A straight-forward single-iteration approach is for each machine to optimize its own local
objective, obtaining
\begin{equation}
  \label{eq:hatwi}
 \hat{w}_i=\arg\min \phi_i(w),
\end{equation}
and then to compute their average:
\begin{equation}
  \label{eq:barw}
  \bar{w}=\frac{1}{m} \sum_{i=1}^m \hat{w}_i.
\end{equation}
This approach, which we refer to as ``one-shot parameter averaging'', was recently considered in \citet{ZiWeSmLi10}  and further analyzed by \citet{zhang2013communication}. The latter also proposed a bias-corrected improvement which perturbs each $\hat{w}_i$ using the optimum on a bootstrap sample. This approach gives only an approximate minimizer of
$\phi(w)$ with some finite suboptimality, rather then allowing us
converge to $\hat{w}$ (i.e.~to obtain solutions with any desired
suboptimality $\epsilon$).  Although approximate solutions are often sufficient for stochastic optimization, we prove in Section
\ref{sec:oneshot} that the one-shot solution $\bar{w}$ can be much worse in terms of minimizing the population objective $F(w)$, compared to the actual
empirical minimizer $\hat{w}$.  It does not seem possible to address
this suboptimality by more clever averaging, and instead additional
rounds of communications appear necessary.


\paragraph{Gradient Descent} One possible multi-round approach to
distributed optimization is a distributed implementation of
gradient descent: at each iteration each machine calculates $\nabla
\phi_i(w^{(t)})$ at the current iterate $w^{(t)}$, and then these are
averaged to obtain the overall gradient $\nabla \phi(w^{(t)})$, and a
gradient step is taken.  As the iterates are then standard gradient
descent iterates, the number of iterations, and so also number of
communication rounds, is linear in the conditioning of the
problem -- or, if accelerated gradient descent is used, proportional to
the square root of the condition number: If $\phi(w)$ is
$L$-smooth and $\lambda$-strongly convex, then
\begin{equation}
  \label{eq:gdruntime}
 \Ocal\left(\sqrt{\frac{L}{\lambda}}\log\left(\frac{1}{\epsilon}\right)\right)
\end{equation}
iterations are needed to attain an $\epsilon$-suboptimal
solution.  The polynomial dependence on the condition number may be
disappointing, as in many problems the parameter of strong convexity
$\lambda$ might be very small.  E.g., when strong convexity arises
from regularization, as in many stochastic optimization problems,
$\lambda$ decreases with the overall sample size $N=nm$, and is
typically at most $1/\sqrt{nm}$ (\citealt{sridharan2008fast,SSSSS09};
and see also Section \ref{sec:reg} below).  The number of iterations /
communication rounds needed for distributed accelerated gradient
descent then scales as $\sqrt[4]{nm}$, i.e.~increases polynomially
with the sample size.

Instead of gradient descent, one may also consider more sophisticated
methods which utilize gradient information, such as quasi-Newton
methods. For example, a distributed implementation using L-BFGS has
been proposed in \cite{AgChDuLa11}.  However, no guarantee better then
\eqref{eq:gdruntime} can be ensured for gradient-based
methods \cite{YudNem83}, and we thus may still get a polynomial dependence on the sample size.

\paragraph{ADMM and other approaches} Another popular approach is
distributed alternating direction method of multipliers (ADMM,
e.g.~\citealt{BoPaChPeEc11}), where the machines alternate between
computing shared dual variables in a distributed manner, and solving
augmented Lagrangian problems with respect to their local data.
However, the convergence of ADMM can be slow. Although recent works
proved a linear convergence rate under favorable assumptions
\cite{DeYi12,HoLu12}, we are not aware of any analysis where the
number of iterations / communication rounds doesn't scale strongly
with the condition number, and hence the sample size, for learning
applications. A similar dependence occurs with other recently-proposed
algorithms for distributed optimization \citep[e.g.][]{Ya13,MaKeSyBo13,DeGiShXi12,CoShSrSr11,DuAgWa12}.
We also mention that our framework is orthogonal to much recent work
on distributed coordinate descent methods
\citep[e.g.][]{ReReWrNi11,RiTa13}, which assume the data is split feature-wise
rather than instance-wise.

\paragraph{Our Method} The method we propose can be viewed as an
approximate Newton-like method, where at each iteration, instead of a
gradient step, we take a step appropriate for the geometry of the
problem, {\em as estimated on each machine separately}. In
particular, for quadratic objectives, the method can be seen as taking approximate Newton steps, where each machine $i$ implicitly uses its local Hessian $\nabla^2 \phi_i(w)$ (although no Hessians are explicitly computed!). Unlike ADMM, our method can take advantage of the fact that for machine learning applications, the sub-problems are usually similar: $\phi_i \approx \phi$.  \removed{When $n$ is large, our approach leads to a
  fast Newton-like convergence rate that is superior to ADMM.} We
refer to our method as DANE---Distributed Approximate NEwton.

DANE is applicable to any smooth and strongly
convex problem.  However, as is typical of Newton and Newton-like
methods, its generic analysis is not immediately apparent.  For
general functions, we can show convergence, but cannot rigorously prove improvement over gradient descent.  Instead, in
order to demonstrate DANE's advantages and give a sense of its
benefits, we focus our theoretical analysis on quadratic objectives.  For stochastic quadratic
objectives, where $f(w,z)$ is $L$-smooth and $\lambda$-strongly convex
in $w\in\reals^d$, we show that
\begin{equation}
  \label{eq:daneintroruntime}
  \Ocal\left( \frac{(L/\lambda)^2}{n} \log(dm)\log(\frac{1}{\epsilon}) \right)
\end{equation}
iterations are sufficient for DANE to find $\tilde{w}$
such that with high probability
$\hat{F}(\tilde{w})\leq\hat{F}(\hat{w})+\epsilon$.  When $L/\lambda$
is fixed and the number of examples $n$ per machine is large (the regime considered by \citealt{zhang2013communication}), \eqref{eq:daneintroruntime}
establishes convergence after a \emph{constant} number of iterations / communication rounds.
When $\lambda$ scales as $1/\sqrt{nm}$, as discussed above,
\eqref{eq:daneintroruntime} yields convergence to the empirical
minimizer in a number of iterations that scales roughly linearly
with the number of machines $m$, but {\em not} with the
sample size $N=nm$.  To the best of our knowledge, this is the first
algorithm which provably has such a behavior. We also provide evidence for similar behavior on non-quadratic objectives.


\paragraph{Notation and Definitions} For vectors, $\norm{v}$ is always
the Euclidean norm, and for matrices $\norm{A}_2$ is the spectral
norm.  We use $\lambda \preccurlyeq A \preccurlyeq L$ to indicate
that the eigenvalues of $A$ are bounded between $\lambda$ and $L$.  We
say that a twice differentiable\footnote{All our results hold also for
  weaker definitions of smoothness and strong convexity which do not
  require twice differentiability.} function $f(w)$ is
$\lambda$-strongly convex or $L$-smooth, iff for all $w$, its Hessian
is bounded from below by $\lambda$ (i.e. $\lambda \preccurlyeq \nabla^2 f(w)$), or above
by $L$ (i.e. $\nabla^2 f(w) \preccurlyeq L$) respectively.

\section{Stochastic Optimization and One-shot Parameter Averaging}\label{sec:oneshot}

In a stochastic optimization setting, where the true objective is the
population objective $F(w)$, there is a limit to the accuracy with
which we can minimize $F(w)$ given only $N=nm$ samples, even using the
exact empirical minimizer $\hat{w}$.  It is thus reasonable to compare
the suboptimality of $F(w)$ when using the exact $\hat{w}$ to what can be
attained using distributed optimization with limited communication.

When $f(w,z)$ has gradients with bounded second moments, namely when
$\forall_w \E_z\left[\norm{\nabla_w f(w,z)}^2\right]\leq
G^2$, and $F(w)$ is $\lambda$-strongly convex, then \citep{SSSSS09}\footnote{More
precisely, \citep{SSSSS09} shows this assuming $\norm{\nabla_w f(w,z)}^2\leq G^2$
for all $w,z$, but the proof easily carries over to this case.}
\begin{equation}
  \label{eq:Fwerm}
  \E[ F(\hat{w}) ] \leq F(w^*) + \Ocal\left(\frac{G^2}{\lambda N}\right)
   = \inf_w F(w) + \Ocal\left(\frac{G^2}{\lambda nm}\right)
\end{equation}
where $w^*=\arg\min F(w)$ is the population minimizer and the
expectation is with respect to the random sample of size $N=nm$.  One
might then ask whether a suboptimality of $\epsilon =
\Ocal\left(\frac{G^2}{\lambda nm}\right)$ can be also be achieved using a few,
perhaps only one, round of communication.  This is different from
seeking a distributed optimization method that achieves any
arbitrarily small empirical suboptimality, and thus converges to
$\hat{w}$, but might be sufficient in terms of stochastic
optimization.

For one-shot parameter averaging, \citet[Corollary
2]{zhang2013communication} recently showed that for $\lambda$-strongly
convex objectives, and when moments of the first, second and third
derivatives of $f(w,z)$ are bounded by $G$, $L$, and $M$
respectively\footnote{The exact conditions in
  \citet{zhang2013communication} refer to various high order moments,
  but are in any case satisfied when $\norm{\nabla_w f(w,z)}\leq G$,
  $\norm{\nabla^2_w f(w,z)}_2 \leq L$ and $\nabla^2 f(w,z)$ is
  $M$-Lipschitz in the spectral norm.  For learning problems, all derivatives
  of the objective can be bounded in terms of a bound on the data and
  bounds on the derivative of a scalar loss function, and are less of
  a concern to us.}, then
\begin{equation}
  \label{eq:oneshotParamError}
  \E{ \norm{\bar{w}-w^*}^2 } \leq \tilde{\Ocal}\left( \frac{G^2}{\lambda^2
      nm} + \frac{G^4 M^2}{\lambda^6 n^2} + \frac{L^2 G^2 \log
      d}{\lambda^4 n^2} \right),
\end{equation}
where $\bar{w}$ is the one-shot average estimator defined in \eqref{eq:barw}.
This implies that the population suboptimality $\E[ F(\bar{w})] - F(w^*)$ is bounded by
\begin{equation}
  \label{eq:oneshotSubopt}
  \tilde{\Ocal}\left( \frac{L G^2}{\lambda^2
      nm} + \frac{L G^4 M^2}{\lambda^6 n^2} + \frac{L^3 G^2 \log
      d}{\lambda^4 n^2} \right).
\end{equation}
\citet{zhang2013communication} argued that the dependence on the
sample size $mn$ above is essentially optimal: the dominant term (as
$n\rightarrow\infty$, and in particular when $n \gg m$) scales as
$1/(nm)$, which is the same as for the empirical minimizer $\hat{w}$
(as in eq. \ref{eq:Fwerm}), and so one-shot parameter averaging
achieves the same population suboptimality rate, using only a single
round of communication, as the best rate we can hope for using
unlimited communication, or if all $N=nm$ samples were on the same
machine. Moreover, the $\Ocal(n^{-2})$ terms can be replaced by a $\Ocal(n^{-3})$ term using an appropriate bias-correction procedure.

However, this view ignores the dependence on the other parameters, and
in particular the strong convexity parameter $\lambda$, which is much
worse in \eqref{eq:oneshotSubopt} relative to \eqref{eq:Fwerm}.  The
strong convexity parameter often arises from an explicit
regularization, and decays as the sample size increases.  E.g., in
regularized loss minimization and SVM-type problems
\citep{sridharan2008fast}, as well as more generally for stochastic
convex optimization \citep{SSSSS09}, the regularization parameter, and
hence the strong convexity parameter, decreases as
$\frac{1}{\sqrt{N}}=\frac{1}{\sqrt{nm}}$.  In practice, $\lambda$ is
often chosen even smaller, possibly as small as $\frac{1}{N}$.
Unfortunately, substituting $\lambda=\Ocal(1/\sqrt{nm})$ in
\eqref{eq:oneshotSubopt} results in a useless bound, where even the first term does not decrease with the sample size.

Of course, this strong dependence on $\lambda$ might be an
artifact of the analysis of \citeauthor{zhang2013communication}.  However,
in Theorem \ref{thm:lowbound} below, we show that even in a simple
one-dimensional example, when $\lambda\leq \Ocal(1/\sqrt{n})$, the population
sub-optimality of the one-shot estimator (using $m$ machines
and a total of $nm$ samples), can be no better then the population
sub-optimality using just $n$ samples, and much worse than what can be
attained using $nm$ samples.  In other words, one-shot averaging does not
give any benefit over using only the data on a single machine, and
ignoring all other $(m-1)n$ data points.

\begin{theorem}\label{thm:lowbound}
For any per-machine sample size $n\geq 9$, and any $\lambda \in \left(0,\frac{1}{9\sqrt{n}}\right)$, there exists a distribution $\Dcal$ over examples and a stochastic optimization problem on a convex set\footnote{Following the framework of \citeauthor{zhang2013communication}, we present
an example where the optimization is performed on a bounded set, which ensures that the gradient moments are bounded.
However, this is not essential and the same result can be shown
when the domain of optimization is $\reals$.} $\Wcal\subset \reals$, such that:
\begin{itemize}
    \item $f(w;z)$ is $\lambda$-strongly convex, infinitely differentiable, and $\forall_{w\in\Wcal} \E_z[\norm{\nabla f(w;z)}^2]\leq 9$.
    \item For any number of machines $m$, if we run one-shot parameter averaging to compute $\wavg$, it holds for some universal constants $C_1,C_2,C_3,C_4$ that
\[
\E[\norm{\wavg-\wopt}^2] \geq \frac{C_1}{\lambda^2 n}~,~\E[\norm{\werm-\wopt}^2] \leq \frac{C_2}{\lambda^2 nm}
\]
\[
\E[F(\wavg)]-F(\wopt) \geq \frac{C_3}{\lambda n}~,~\E[F(\werm)]-F(\wopt) \leq \frac{C_4}{\lambda nm}
\]
\end{itemize}
\end{theorem}
The intuition behind the construction of Theorem \ref{thm:lowbound} is that
when $\lambda$ is small, the deviation of each machine output $\hat{w}_i$
from $\wopt$ is large, and its expectation is biased away from $\wopt$. The
exact bias amount is highly problem-dependent, and cannot be eliminated by
any fixed averaging scheme. Since bias is not reduced by averaging, the
optimization error does not scale down with the number of machines $m$. The
full construction and proof appear in appendix \ref{app:lowbound}. In the
appendix we also show that the bias correction proposed by
\citeauthor{zhang2013communication} to reduce the lower-order terms in
equation \eqref{eq:oneshotParamError} does not remedy this problem.

\section{Distributed Approximate Newton-type Method}\label{sec:alg}

We now describe a new iterative method for distributed optimization.
The method performs two distributed averaging computations per iteration,
and outputs a predictor $w^{(t)}$ which, under suitable parameter
choices, converges to the optimum $\hat{w}$.  The method, which we
refer to as DANE ({\bf D}istributed {\bf A}pproximate {\bf NE}wton-type
Method) is described in Figure~\ref{fig:dane}.

\begin{figure*}[htbp]
\begin{myalgo}{Procedure DANE}
\textbf{Parameter}:  learning rate $\eta >0$ and regularizer $\mu>0$\\
\textbf{Initialize}: Start at some $w^{(0)}$, e.g.~$w^{(0)}=0$ \\
\textbf{Iterate}: for $t=1,2,\dots$ \+ \\
Compute $\nabla \phi(w^{(t-1)})=\frac1m \sum_{i=1}^m \nabla \phi_i(w^{(t-1)})$ and distribute to all machines\\
For each machine $i$, solve
$w_i^{(t)} = \arg\min_{w} \left[\phi_i(w) - (\nabla \phi_i(w^{(t-1)})- \eta \nabla \phi(w^{(t-1)}))^\top w + \frac{\mu}{2} \|w-w^{(t-1)}\|_2^2\right]$  \\
Compute $w^{(t)}= \frac1m \sum_{i=1}^m w_i^{(t)}$ and distribute to all machines \hspace{0.5in} $(*)$ \- \\
\textbf{end}
\end{myalgo}
\caption{Distributed Approximate NEwton-type method (DANE)}
\label{fig:dane}
\end{figure*}

DANE maintains an agreed-upon iterate $w^{(t)}$, which is synchronized
among all machines at the end of each iteration.  In each iteration,
we first compute the gradient $\nabla \phi(w^{(t-1)})$ at the current
iterate, by averaging the local gradients $\nabla \phi_i(w^{(t-1)})$.
Each machine then performs a
separate local optimization, based on its own local objective
$\phi_i(w)$ and the computed global gradient $\nabla \phi(w^{(t)})$,
to obtain a local iterate $w_i^{(t)}$.  These local iterates are
averaged to obtain the
centralized iterate $w^{(t)}$.

The crux of the method is the local optimization performed on each
machine at each iteration:
\begin{align}
&w_i^{(t)} = \arg\min_{w} [\phi_i(w)    \label{eq:wiupdate}\\
&-(\nabla \phi_i(w^{(t-1)})- \eta \nabla \phi(w^{(t-1)}))^\top w + \frac{\mu}{2} \|w-w^{(t-1)}\|_2^2]
\notag
\end{align}
To understand this local optimization, recall the definition of the
Bregman divergence corresponding to a strongly convex function $\psi$:
\[
D_{\psi}(w';w) = \psi(w') - \psi(w) - \inner{\nabla \psi(w),w'-w} .
\]
Now, for each local objective $\phi_i$, consider the regularized local
objective, defined as
\[
h_i(w) = \phi_i(w) + \frac{\mu}{2} \norm{w}^2
\]
and its corresponding Bregman divergence:
\[
D_i(w'; w) = D_{h_i}(w';w) = D_{\phi_i}(w';w) + \frac{\mu}{2} \norm{w'-{w}}^2 .
\]
It is not difficult to check that the local optimization problem
\eqref{eq:wiupdate} can be written as
\begin{align}
  w_i^{(t)} =
\arg\min_{w} \phi(w^{(t-1)}) &+ \inner{\nabla \phi(w^{(t-1)}),w-w^{(t-1)}}\notag \\ &+ \frac1\eta D_i(w;w^{(t-1)}),  \label{eq:wiMDupdate}
\end{align}
where we also added the terms $\phi(w^{(t-1)})+\inner{\nabla
  \phi(w^{(t-1)}),w^{(t-1)}}$ which do not depend on $w$ and do not
affect the optimization.  The first two terms in \eqref{eq:wiMDupdate}
are thus a linear approximation of the overall objective $\phi(w)$
about the current iterate $w^{(t-1)}$, and do not depend on the
machine $i$.  What varies from machine to machine is the potential
function used to localize the linear approximation.  The update
\eqref{eq:wiMDupdate} is in-fact a mirror descent update
\citep{nemirovski1978cesaro,BeTe03} using the potential function
$h_i$, and step size $\eta$.

Let us examine this form of update. When $\mu\rightarrow\infty$ the potential function essentially becomes a squared Euclidean norm, as in gradient descent updates.  In fact, when $\eta,\mu\rightarrow\infty$ as
$\tilde{\eta}\doteq\frac{\eta}{\mu}$ remains constant, the update
\eqref{eq:wiMDupdate} becomes a standard gradient descent update on
$\phi(w)$ with stepsize $\tilde{\eta}$.  In this extreme, the update
does not use the local objective $\phi_i(w)$, beyond the centralized
calculation of $\nabla\phi(w)$, the updates \eqref{eq:wiMDupdate} are
the same on all machines, and the second round of communication is not
needed.  DANE reduces to distributed gradient descent, with its
iteration complexity of $\Ocal\left(\frac{L}{\lambda}\log(1/\epsilon)\right)$.

At the other extreme, consider the case where $\mu=0$ and all local
objectives are equal, i.e.~$h_i(w)=\phi_i(w)=\phi(w)$.  Substituting
the definition of the Bregman divergence into \eqref{eq:wiMDupdate},
or simply investigating \eqref{eq:wiupdate}, we can see that
$w_i^{(t)}=\arg\min \phi_i(w)=\arg\min \phi(w)=\hat{w}$.  That is,
DANE converges in a single iteration to the overall empirical
optimum.  This is an ideal Newton-type iteration, where the potential
function is perfectly aligned with the objective.

Of course, if $\phi_i(w)=\phi(w)$ for all machines $i$, we would not
need to perform distributed optimization in the first place.
Nevertheless, as $n\rightarrow\infty$, we can hope that $\phi_i(w)$
are similar enough to each other, such that \eqref{eq:wiMDupdate}
{\em approximates} such an ideal Newton-type iteration, gets us very
close to the optimum, and very few such iterations are sufficient.

In particular, consider the case where $\phi_i(w)$, and hence also
$\phi(w)$ are quadratic.  In this case, the Bregman divergence
$D_i(w;w^{(t-1)})$ takes the form:
\begin{equation}
  \label{eq:Diquad}
\frac12 (w-w^{(t-1)})^\top (\nabla^2 \phi_i(w^{(t-1)})+ \mu I) (w-w^{(t-1)}),
\end{equation}
and the update \eqref{eq:wiMDupdate} can be solved in closed form:
\begin{align}
  w_i^{(t)} &= w^{(t-1)} - \eta (\nabla^2\phi_i(w^{(t-1)})+\mu I)^{-1}
  \nabla \phi(w^{(t-1)}) \notag\\*
  w^{(t)} &= w^{(t-1)} \notag\\&
  - \eta \left(\frac{1}{m} \sum_i
    (\nabla^2\phi_i(w^{(t-1)})+\mu I)^{-1}\right)
  \nabla \phi(w^{(t-1)}).  \label{eq:wiquad}
\end{align}
Contrast this with the true Newton update:
\begin{equation}
  \label{eq:truenewton}
  w^{(t)} = w^{(t-1)} - \eta \left(\frac{1}{m} \sum_i
    \nabla^2\phi_i(w^{(t-1)})\right)^{-1}
  \nabla \phi(w^{(t-1)}).
\end{equation}
The difference here is that in \eqref{eq:wiquad} we approximate the
inverse of the average of the local Hessians with the average of the
inverse of the Hessians (plus a possible regularizer).  Again we see
that the DANE update \eqref{eq:wiquad} approximates the true
Newton update \eqref{eq:truenewton}, which can be performed in a
distributed fashion {\em without communicating the Hessians}.

For a quadratic objective, a single Newton update is enough to find
the exact optimum.  In Section \ref{sec:quad} we rigorously
analyze the effects of the distributed approximation, and quantify the
number of DANE iterations (and thus rounds of communication) required.

For a general convex, but non-quadratic, objective, the standard
Newton approach is to use a quadratic approximation to the ideal Bregman
divergence $D_\phi$.  This leads to the familiar quadratic Newton
update in terms of the Hessian.  DANE uses a different sort of
approximation to $D_\phi$: we use a {\em non-quadratic} approximation,
based on the entire objective and not just a local quadratic
approximation, but approximate the potential on each node separately.
In the stochastic setting, this approximation becomes better and
better, and thus the required number of iterations decrease, as
$n\rightarrow\infty$.

Since it is notoriously difficult to provide good global analysis for
Newton-type methods, we will investigate the global convergence
behavior of DANE carefully in the next Section but only for quadratic
objective functions.  This analysis can also be viewed as indicative
for non-quadratic objectives, as locally they can be approximated by
quadratics and so should enjoy the same behavior, at least
asymptotically.  For non-quadratics, we provide a rigorous convergence
guarantee when the stepsize $\eta$ is sufficiently small or the regularization
parameter $\mu$ is sufficiently large (in Section \ref{sec:gen}).
However, this analysis does not show a benefit over distributed
gradient descent for non-quadratics. We partially bridge this gap by showing that even in the non-quadratic case, the convergence rate improves as the local problems $\phi_i$ become more similar.

\section{DANE for Quadratic Objectives}\label{sec:quad}

In this Section, we analyze the performance of DANE on quadratic objectives.
We begin in Section \ref{sec:quadphi} with an analysis of DANE for arbitrary
quadratic objectives $\phi_i(w)$, without stochastic assumptions, deriving a
guarantee in terms of the approximation error of the true Hessian.  Then in
Section \ref{sec:quadF} we consider the stochastic setting where the
instantaneous objective $f(w,z)$ is quadratic in $w$, utilizing a bound on
the approximation error of the Hessian to obtain a performance guarantee for
DANE in terms of the smoothness and strong convexity of $f(w,z)$ .  In
Section \ref{sec:reg} we also consider the behavior for stochastic
optimization problems, where $\lambda$ is set as a function of the sample
size $N=nm$.

\subsection{Quadratic $\phi_i(w)$}\label{sec:quadphi}

We begin by considering the case where each local objective
$\phi_i(w)$ is quadratic, i.e.~has a fixed Hessian.  The overall
objective $\phi(w)$ is then of course also quadratic.

\begin{theorem}\label{thm:main} After $t$ iterations of DANE on
  quadratic objectives with Hessians $H_i=\nabla^2 \phi_i(w)$, we
  have:
  \[
  \norm{w^{(t)}-\werm} \leq \norm{I - \eta \tilde{H}^{-1} H}_2^t \norm{w^{(0)}-\werm},
  \]
where $H=\nabla^2 \phi(w)=\frac{1}{m}\sum_{i=1}^m H_i$ and $\tilde{H}^{-1}=
\frac{1}{m} \sum_{i=1}^n (H_i+\mu I)^{-1}$.
\end{theorem}
The proof appears in Appendix \ref{app:thmmainproof}. The theorem implies that if $\norm{I - \eta \tilde{H}^{-1} H}_2$ is smaller than $1$,
we get a linear convergence rate. Indeed, we would expect $\norm{I - \eta \tilde{H}^{-1} H}_2 \ll 1$
as long as $\eta$ is close to $1$ and $\tilde{H}$ is a good
approximation for the true Hessian $H$, hence $\tilde{H}^{-1}H
\approx I$. In particular, if
$H$ is not too ill-conditioned, and all $H_i$ are sufficiently close to their
average $H$, we can indeed ensure $\tilde{H}\approx H$.  This is captured by
the following lemma (whose proof appears in Appendix \ref{app:HtildeH}):
\begin{lemma}\label{lem:HtildeH}
  If $0<\lambda \preccurlyeq H \preccurlyeq L$ and for all $i$,
  $\norm{H_i-H}_2 \leq \beta$, then setting $\eta=1$ and
  $\mu=\max\{0,\frac{8\beta^2}{\lambda}-\lambda\}$, we have:
\[
\norm{I-\tilde{H}^{-1}H}_2 \leq \begin{cases}\frac{4\beta^2}{\lambda^2}&\text{if}~~~\frac{4\beta^2}{\lambda^2}\leq \frac{1}{2}\\1-\frac{\lambda^2}{16\beta^2}&\text{otherwise}.\end{cases}
\]
\end{lemma}

In the next Section, we consider the stochastic setting, where we can
obtain bounds for $\norm{H_i-H}_2$ that improve with the sample size,
and plug these into Lemma \ref{lem:HtildeH} and Theorem
\ref{thm:main} to obtain a performance guarantee for DANE.

\subsection{Stochastic Quadratic Problems}\label{sec:quadF}

We now turn to a stochastic quadratic setting, where
$\phi_i(w)=\hat{F}_i(w)$ as in \eqref{eq:hatFi}, and the instantaneous
losses are smooth and strongly convex quadratics.  That is, for all
$z$, $f(w,z)$ is quadratic in $w$ and $\lambda \preccurlyeq \nabla^2_w
f(w,z) \preccurlyeq L$.

We first use a matrix concentration bound to establish that all
Hessians $H_i=\nabla^2 \hat{F}_i(w)$ are close to each other, and
hence also to their average:
\begin{lemma}\label{lem:hoef}
  If $0 \preccurlyeq \nabla^2_w f(w,z) \preccurlyeq L$ for all $z$, then
  with probability at least $1-\delta$ over the samples, $\max_i
  \norm{H_i-H}_2 \leq \sqrt{\frac{32 L^2 \log(dm/\delta)}{n}}$, where
  $H_i=\nabla^2 \hat{F}_i(w)$ and $H=\nabla^2 \hat{F}(w)$.
\end{lemma}
The proof appears in appendix \ref{app:lemhoefproof}. Combining Lemma \ref{lem:hoef}, Lemma \ref{lem:HtildeH} and Theorem \ref{thm:main}, we can conclude:

\begin{theorem}\label{thm:stoch}
  In the stochastic setting, and when the instantaneous losses are
  quadratic with $\lambda \preccurlyeq \nabla f(w,z) \preccurlyeq L$,
  then after
$$ t= \Ocal\left( \frac{(L/\lambda)^2}{n} \log\left(\frac{dm}{\delta}\right) \log\left(\frac{L\norm{w^0-\werm}^2}{\epsilon}\right)\right)$$ iterations of DANE, we
have, with probability at least $1-\delta$, that
$\hat{F}(w^{(t)})\leq\hat{F}(\werm)+\epsilon$.
\end{theorem}
The proof appears in Appendix \ref{app:thmstochproof}. From the theorem, we see that if the condition number $L/\lambda$ is fixed, then as
$n\rightarrow\infty$ the number of required iterations decreases.  In
fact, for any target sub-optimality $\epsilon$, as long as the sample
size is at least logarithmically large, namely
$n=\Omega\left((L/\lambda)^2\log(dm)\log(\frac{1}{\epsilon})\right)$,
we can obtain the desired accuracy after a constant or even a single DANE iteration! This is a mild requirement on the sample size, since $N$ generally increases at least linearly with $1/\epsilon$.

We next turn to discuss the more challenging case where the condition
number decays with the sample size.

\subsection{Analysis for Regularized Objectives}\label{sec:reg}

Consider a stochastic convex optimization scenario where the
instantaneous objectives $f(w,z)$ are not strongly convex.  For
example, this is the case in linear prediction (including linear and
kernel classification and regression, support vector machines, etc.),
and more generally for generalized linear objectives of the form
$f(w,z)=\ell_z(\inner{w,\Psi(z)})$.  For such generalized linear
objectives, the Hessian $\nabla_w^2 f(w,z)$ is rank-$1$, and so
certainly not strongly convex, even if $\ell_z(\cdot)$ is strongly convex.

Confronted with such non-strongly-convex objectives, a standard
approach is to perform empirical
minimization on a {\em regularized} objective \citep{SSSSS09}.  That
is, to define the regularized instantaneous objective
\begin{equation}
  \label{eq:flambda}
  f_\lambda(w,z)=f(w,z)+\frac{\lambda}{2}\norm{w}^2
\end{equation}
and minimize the corresponding empirical objective $\hat{F}_\lambda$.
The instantaneous objective $f_\lambda(w,z)$ of the modified
stochastic optimization problem is now $\lambda$-strongly convex.
If $f(w,z)$ are $G$-Lipschitz in $w$, then we have \citep{SSSSS09}:
\begin{align*}
  F(\werm_\lambda)
  &\leq F_\lambda(\werm_\lambda) \leq F_\lambda(w_\lambda^*) +
  \Ocal\left(\frac{G^2}{\lambda N}\right) \\
  & = \inf_w \left( F(w) + \frac{\lambda}{2}\norm{w}^2 \right) +
  \Ocal\left(\frac{G^2}{\lambda N}\right) \\
  &\leq \inf_{\norm{w}\leq B} F(w) + \Ocal\left(\lambda B^2 +
  \frac{G^2}{\lambda N}\right),
\end{align*}
where $\werm_\lambda=\arg\min \hat{F}_\lambda(w)$ and
$w_\lambda^* = \arg\min F_\lambda(w)$.  The optimal choice of
$\lambda$ in the above is $\lambda = \sqrt{\frac{G^2}{B^2
    N}}$, where $B$ is a bound on the predictors we would like to
compete with, and with this $\lambda$ we get the optimal rate:
\begin{equation}
  F(\werm_\lambda) \leq \inf_{\norm{w}\leq B} F(w) + \Ocal\left(\sqrt{\frac{B^2 G^2}{N}}\right).
\end{equation}

It is thus instructive to consider the behavior of DANE when $\lambda
= \Theta\left(\sqrt{\frac{G^2}{B^2
    N}}\right) = \Theta\left(\sqrt{\frac{G^2}{B^2
    nm}}\right)$.  Plugging this choice of $\lambda$ into Theorem \ref{thm:stoch}, we get that the number of DANE iterations behaves as:
\begin{equation}
  \Ocal\left( \frac{L^2 B^2}{G^2} \cdot m \cdot \log(dm) \log(1/\epsilon) \right).
\end{equation}
That is, unlike distributed gradient descent, or any other relevant method we are aware of, the number of required iterations / communication rounds does {\em not} increase with the sample size, and only scales linearly with the
number of machines.

\section{Convergence Analysis for Non-Quadratic Objectives}\label{sec:gen}

As discussed above, it is notoriously difficult to obtain generic global
analysis of Newton-type methods.  Our main theoretical result in this paper
is the analysis for quadratic objectives, which we believe is also
instructive for non-quadratics.  Nevertheless, we complement this with a
convergence analysis for generic objectives.

We therefore return to considering generic convex
objectives $\phi_i(w)$.  We also do not make any stochastic
assumptions.  We only assume that each $\phi_i(w)$ is $L_i$-smooth and
$\lambda_i$ strongly convex, and that the combined objective $\phi(w)$
is $L$-smooth and $\lambda$-strongly convex.

\begin{theorem}
  Assume that for all $i,w,z$, $\lambda_i \preccurlyeq \nabla^2
  \phi_i(w) \preccurlyeq L_i$ and $\lambda \preccurlyeq \nabla^2
  \phi(w) \preccurlyeq L$. Let $$\rho = \frac{1}{m}
  \sum_{i=1}^m \left[\frac{1}{\mu+L_i} - \frac{\eta
      L}{2(\mu+\lambda_i)^2}\right] \eta \lambda.$$  If $\rho >0$, then the DANE iterates
      satisfy
 $
 \phi(w^{(t)}) -\phi(\werm) \leq (1-\rho)^t [\phi(w^{(0)}) - \phi(\werm)] .
 $
\label{thm:general-loss}
\end{theorem}
The proof appears in Appendix \ref{app:general-loss}. The theorem establishes that with any $\mu>0$ and
small enough step-size $\eta$, DANE converges to $\werm$.  If each
$\phi_i(w)$ is strongly convex, we can also take $\mu=0$ and
sufficiently small $\eta$ and ensure convergence to $\werm$. However, the optimal setting of $\eta$ and $\mu$ above is to take
$\mu\rightarrow\infty$ and set $\eta = \mu/L$, in which case $\rho
\rightarrow \lambda/L$, and we recover distributed gradient descent,
with the familiar gradient descent guarantee.

\begin{figure}[t]
\centering
\includegraphics[scale=0.65,trim=1cm 0.5cm 1cm 0.1cm,clip=true]{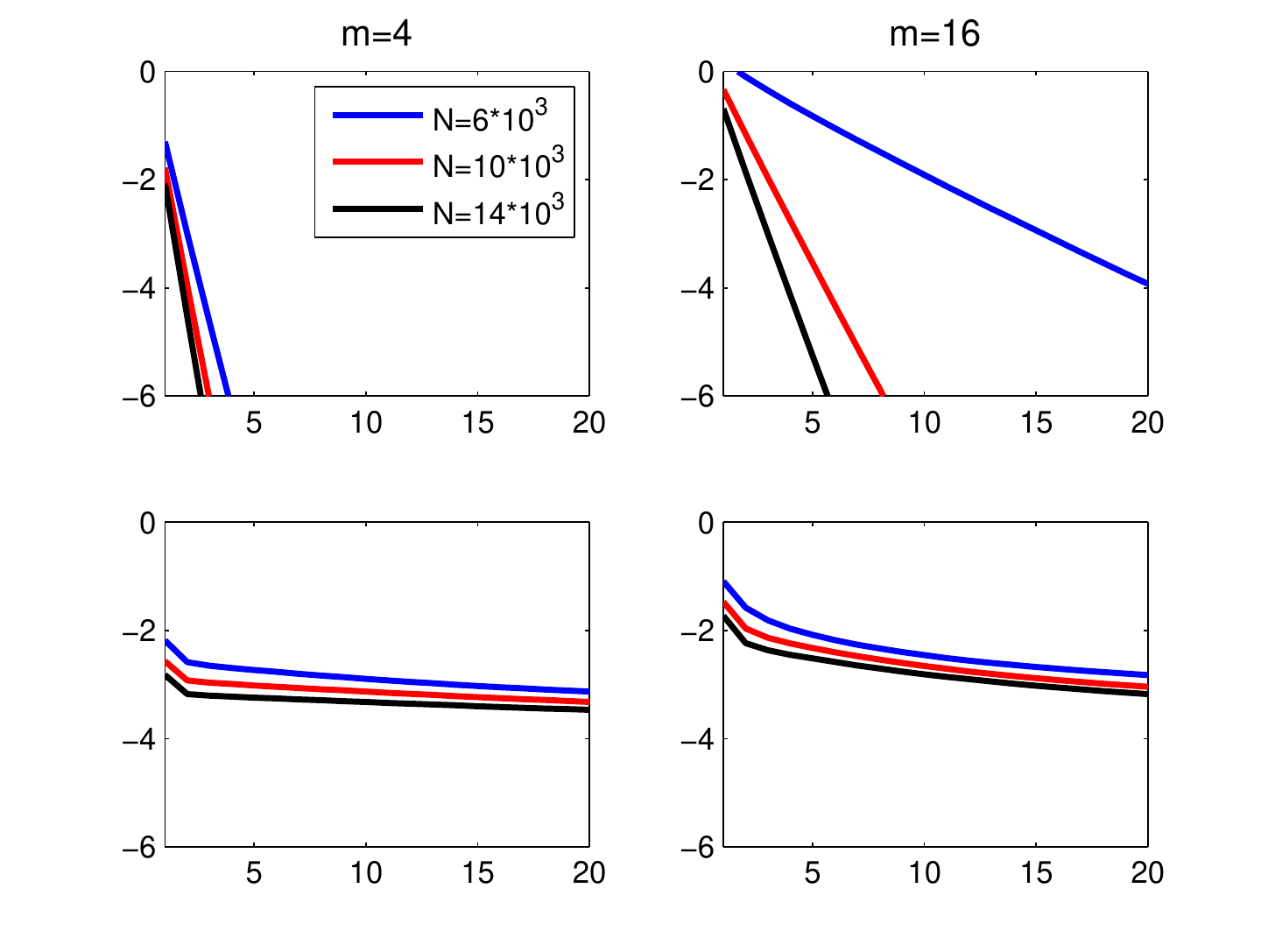}
\caption{Synthetic dataset: Convergence rate for different number of machines $m$ and sample sizes $N$. The top row presents results for DANE, and the bottom row for ADMM. The $x$-axis is the iteration number,
and the $y$-axis is the logarithm (in base 10) of the suboptimality.}
\label{fig:synthetic}
\end{figure}

We again emphasize that the analysis above is weak and does not take into
account the relationship between the local objectives $\phi_i(w)$. We believe
that the quadratic analysis of Section \ref{sec:quad} better captures the
true behavior of DANE. Moreover, we can partially bridge this gap by the
following result, which shows that a variant of DANE enjoys a linear
convergence rate which improves as the local objectives $\phi_i$ become more
similar to $\phi$ (the proof is in Appendix \ref{app:general-loss-refined}):
\begin{theorem}
  Assume that in the DANE procedure, we replace step $(*)$ by $w^{(t)}= w_1^{(t)}$, and define $h(\cdot)=h_1(\cdot)$.
  If there exists $\gamma>0$ such that $\forall w,w'$, we have $\gamma D_h(w;w') \leq D_\phi(w;w') \leq \eta^{-1} D_h(w;w')$, then
\[
D_h(\werm;w^{(t)}) \leq (1-\eta \gamma)^t D_h(\werm;w^{(0)}) .
\]
\label{thm:general-loss-refined}
\end{theorem}
\vskip -0.7cm If $\mu$ is small and $\phi_i \approx \phi$, then we expect
$\gamma \approx 1$ and $\eta \approx 1$. In this case, $\eta \gamma \approx
1$, leading to fast convergence.

\section{Experiments}

\begin{figure*}[t]
\centering
\begin{tabular}{|c||c|c|c|c|c|c||c|c|c|c|c|c||c|c|c|c|c|c|}
\hline
&\multicolumn{6}{c||}{COV1} & \multicolumn{6}{c||}{ASTRO} & \multicolumn{6}{c|}{MNIST-47}\\\hline
$m$&2&4&8&16&32&64&2&4&8&16&32&64&2&4&8&16&32&64\\\hline\hline
$\mu=0$&2&2&2&2&2&3&6&6&6&6&12&*&5&5&5&5&6&*\\
$\mu=3\lambda$&9&9&9&9&9&9&14&14&14&14&14&14&10&10&10&10&10&10\\
ADMM&3&3&5&9&16&31&24&20&16&16&14&20&23&23&27&21&31&28\\\hline
\end{tabular}
\caption{Number of iterations required to reach $<10^{-6}$ accuracy on 3 datasets, for varying number of machines $m$.
  Results are for DANE using $\eta=1$ and $\mu=0,\lambda,3\lambda$, and for ADMM. * Indicates non-convergence after $100$ iterations.}
\label{fig:train}
\end{figure*}

\begin{figure*}
\centering
\includegraphics[scale=0.7,trim=0cm 0.5cm 0cm 0cm,clip=true]{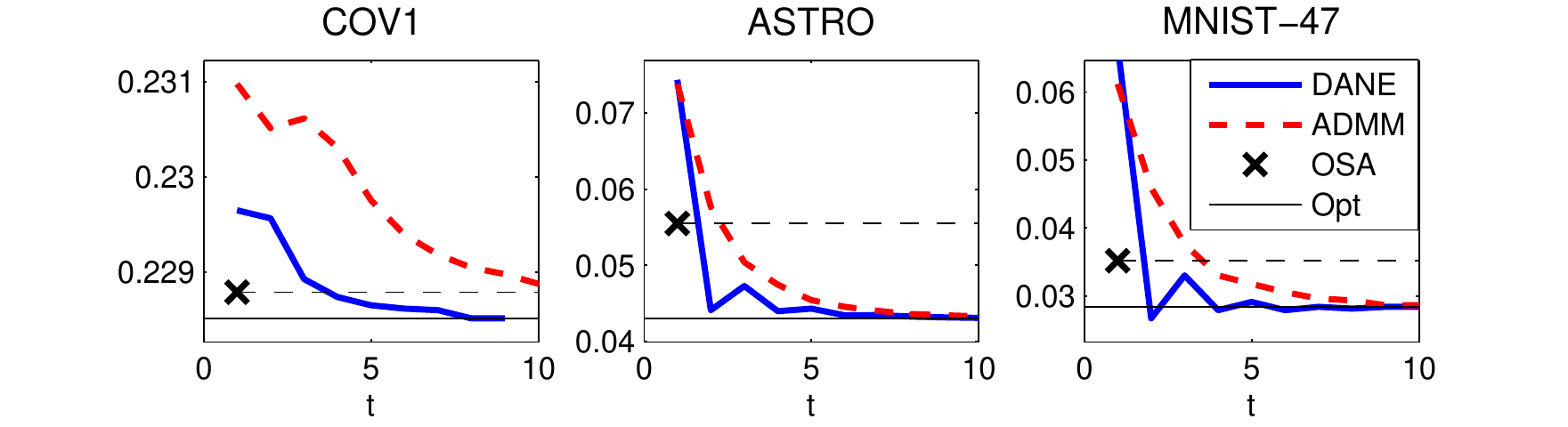}
\caption{Average regularized smooth-hinge loss on the test set as a function of the iteration number. OSA represents bias-corrected one-shot parameter averaging, which requires a single iteration. `Opt' is the average loss of the exact regularized loss minimizer.}
\label{fig:test}
\end{figure*}

In this section, we present preliminary experimental results on our proposed
method. In terms of tuning $\eta,\mu$, we discovered that simply picking
$\eta=1,\mu=0$ (which makes DANE closest to a Newton-type iteration, as
discussed in Section \ref{sec:alg}) often results in the fastest convergence.
However, in unfavorable situations (such as when the data size per machine is
very small), this can also lead to non-convergence. In those cases,
convergence can be recovered by slightly increasing $\mu$ to a small positive
number. In the experiments, we considered $\mu=0,3\lambda$. These are
considerably smaller than what our theory indicates, and we leave the
question of the best parameter choice to future research.

We begin by considering a simple quadratic problem using a synthetic dataset,
where all parameters can be explicitly controlled. We generated $N$ i.i.d.
training examples $(x,y)$ according to the model $y = \inner{x,w^*}+\xi~,~
x\sim \Ncal(0,\Sigma),\xi\sim \Ncal(0,1)$, where $x\in \reals^{500}$, the
covariance matrix $\Sigma$ is diagonal with $\Sigma_{i,i}=i^{-1.2}$, and
$w^*$ is the all-ones vector. Given a set of examples $\{x,y\}$ which is
assumed to be randomly split to different machines, we then solved a standard
ridge regression problem of the form
$\min_{w}\frac{1}{N}\sum_{i=1}^{N}(\inner{x,w}-y)^2+0.005 w^2$, using DANE
(with $\eta=1,\mu=0$). \figref{fig:synthetic} shows the convergence behavior
of the algorithm for different number of machines $m$ as the total number of
examples $N$ (and hence also the data size per machine) increases. For
comparison, we also implemented distributed ADMM \cite{BoPaChPeEc11}, which
is a standard method for distributed optimization but does not take advantage
of the statistical similarity between problems at different machines. The
results for DANE clearly indicate a linear convergence rate, and moreover,
that the rate of convergence \emph{improves} with the data size, as predicted
by our analysis. In contrast, while more data improves the ADMM accuracy
after a fixed number of iterations, the convergence rate is slower and does
not improve with the data size\footnote{To be fair, ADMM performs a single
distributed averaging computation per iteration, while DANE performs two.
However, counting iterations is a more realistic measure of performance,
since both methods also perform a full-scale local optimization at each
iteration.}.

We now turn to present results for solving a smooth non-quadratic problem, this time using non-synthetic datasets. Specifically, we solved a regularized loss minimization problem of the form
$\min_{w}\frac{1}{N}\sum_{i=1}^{N}\ell(y_i\inner{x_i,w})
+\frac{\lambda}{2}\norm{w}^2,
$
where $\ell$ is the smooth hinge loss (as in \cite{SSZ13}) and
the training examples $\{(x_i,y_i)\}$ are randomly split among different machines. We experimented on $3$ datasets: \textsc{COV1} and \textsc{ASTRO-PH} (as used in e.g. \cite{SSZ13,RSR12}),
as well as a subset of the \textsc{MNIST} digit recognition dataset which focuses
on discriminating the 4 from the 7 digits\footnote{We used $\lambda=10^{-5}$ for \textsc{COV1}, $\lambda = 0.0005$ for \textsc{ASTRO} and $\lambda=0.001$ for \textsc{MNIST-47}. For \textsc{MNIST-47}, we randomly chose 10,000 examples as the training set, and the rest of the examples as a test set.}.
In figure \ref{fig:train}, we present the number of iterations required for DANE to reach accuracy $<10^{-6}$ for $\eta=1$ and
$\mu=0,3\lambda$, and for different number of machines. We also report results for ADMM on the same datasets. As in the synthetic case, DANE explicitly takes advantage of the similarity between problems on different machines, and we indeed observe that it tends to converge in less iterations than ADMM. Finally, note that for $\mu=0$ and many machines (i.e.
few data points per machine), DANE may not converge, and increasing $\mu$ fixes this at the cost of slowing down the average convergence rate.

Finally, we examine the convergence on these datasets in terms of the average
loss on the test set. In figure \ref{fig:test}, we present the results for
$m=64$ machines on the three datasets, using DANE (with $\mu=3\lambda$) and
ADMM. We also present for comparison the objective value obtained using
one-shot parameter averaging (OSA), using bias correction as proposed in
\cite{zhang2013communication}. The figure highlights the practical importance
of multi-round communication algorithms: while DANE and ADMM converge to the
value achieved by the regularized loss minimizer, the single-round OSA
algorithm may return a significantly suboptimal result.

\paragraph{Acknowledgements:}{Ohad Shamir and Nathan Srebro are supported by the Intel ICRI-CI
Institute. Ohad Shamir is further supported by an Israel Science Foundation
grant 425/13 and an FP7 Marie Curie CIG grant.}

%

\bibliographystyle{plainnat}
\bibliography{mybib}

\appendix

\section{Lower Bounds for One-shot Parameter Averaging}\label{app:lowbound}

\subsection{Proof of \thmref{thm:lowbound}}

Before providing the proof details, let us first describe the high-level
intuition of our construction. Roughly speaking, one-shot averaging works
well when the bias of the predictor returned by each machine (as a random
vector in Euclidean space, based on the sampled training data) is much
smaller than the variance. Since each such predictor is based on independent
data, averaging $m$ such predictors reduces the variance by a factor of $m$,
leading to good guarantees. However, averaging has no effect on the bias, so
this method is ineffectual when the bias dominates the variance. The
construction below shows that when the strong convexity parameter is small,
this can indeed happen.

More specifically, when the strong convexity parameter is smaller than
$\Ocal(1/\sqrt{n})$, the magnitude of the deviations of the (random)
predictor returned by each machine does not decay with the sample size $n$.
Moreover, its distribution is highly dependent on the data distribution and
the shape of $f$, and is biased in general. Below we use one such
construction, which we found to be convenient for precise analytic
calculations, but the intuition applies much more broadly.

Specifically, let $\Wcal = [-2/\lambda,\log(1/\lambda)]$, and define the loss
function $f(w;z)$ as
\[
f(w;z) = \lambda\left(\frac{1}{2}w^2+\exp(w)\right)-zw.
\]
Furthermore, suppose that $z\sim \Ncal(0,1)$, i.e. the examples have a
standard Gaussian distribution. Note that this function is $\lambda$-strongly
convex, and can be shown to satisfy $\E_z[f'(w;z)^2]\leq 9$ for any
$w\in\Wcal$.

Let $\hat{w}_1$ be the parameter vector returned by the machine $1$ (this is without loss of generality, since all the machines receive examples drawn from the same distribution).
The key to the proof is to show that $\hat{w}_1$ is strongly biased, namely that $\E[\hat{w}_1]$ is bounded away from the true optimum $w^*$.
To compute $\E[\hat{w}_1]$, note that $\hat{w}_1$ minimizes the random function
\begin{equation}\label{eq:erm}
\frac{1}{n}\sum_{i=1}^{n} f(w;z_i) =  \lambda\left(\frac{1}{2}w^2+\exp(w)\right)-\frac{\tilde{z}}{\sqrt{n}}w,
\end{equation}
where $\tilde{z}=(z_1+\ldots+z_n)/\sqrt{n}$. Note that since $z_1,\ldots,z_n$ are i.i.d. Gaussians, $\tilde{z}$ also has the same Gaussian distribution $\Ncal(0,1)$.

Taking the derivative, equating to zero and slightly manipulating the result, we get that
\[
\lambda\sqrt{n}\left(w+\exp(w)\right)=\tilde{z}.
\]
The function on the left-hand-side is strictly monotonically increasing, and has a range $[-\infty,\infty]$. Thus, for any $\tilde{z}$, there exists a unique root $w(\tilde{z})$. Moreover, as long as $|\tilde{z}|\leq \sqrt{n}$, it's easy to verify that $w(\tilde{z})$ is within our domain $\Wcal = [-2/\lambda,\log(1/\lambda)]$, hence $\hat{w}_1=w(\tilde{z})$. Therefore, letting $p(\cdot)$ denote the standard gaussian distribution of $\tilde{z}$, we have
\begin{align}
\E[\hat{w}_1] &= \int_{x=-\sqrt{n}}^{\sqrt{n}}\hat{w}_1 p(x)dx+\int_{|x|>\sqrt{n}}\hat{w}_1 p(x)dx
\leq \int_{x=-\sqrt{n}}^{\sqrt{n}}w(x) p(x)dx+ \frac{2}{\lambda}\Pr(|\tilde{z}|\geq \sqrt{n})\notag\\
&\leq \int_{x=0}^{\sqrt{n}}\left(w(x)+w(-x)\right)p(x)dx + \frac{2}{\lambda}\sqrt{\frac{1}{n\exp(n)}},\label{eq:wxwmx}
\end{align}
where in the last step we used the symmetry of the distribution of $\tilde{z}$ and a standard Gaussian tail bound.

We now turn to analyze $w(x)+w(-x)$. First, we have by definition
\begin{equation}\label{eq:wxdef}
w(x)+\exp(w(x)) = \frac{x}{\lambda\sqrt{n}}
\end{equation}
for all $x$, and therefore
\begin{equation}\label{eq:wexp}
w(x)+w(-x) = \left(\frac{x}{\lambda\sqrt{n}}-\exp(w(x))\right)+\left(-\frac{x}{\lambda\sqrt{n}}-\exp(w(-x))\right)
\leq -\exp(w(x)).
\end{equation}
Therefore, we have $w(x)+w(-x)\leq 0$ for all $x$. More precisely, considering \eqref{eq:wxdef} and the fact that its left hand size is monotonic in $w(x)$, it's easy to verify that for any $x\geq 0$, we have $w(x)\geq \log\left(\frac{x}{\lambda \sqrt{n}}-\log\left(\frac{x}{\lambda \sqrt{n}}\right)\right)$, and $w(-x)\leq -\frac{x}{\lambda \sqrt{n}}$, so using \eqref{eq:wexp},
\[
w(x)+w(-x) \leq -\exp(w(x)) \leq -\frac{x}{\lambda \sqrt{n}}+\log\left(\frac{x}{\lambda \sqrt{n}}\right).
\]
Since $\log(a)<a/2$ for all $a\geq 0$, this expression is at most $-\frac{x}{2\lambda\sqrt{n}}$. Plugging this back into \eqref{eq:wxwmx}, and using the assumption $n\geq 9$, we get that
\[
\E[\hat{w}_1] \leq -\frac{1}{2\lambda\sqrt{n}}\int_{x=0}^{\sqrt{n}}x p(x)dx + \frac{2}{\lambda}\sqrt{\frac{1}{n\exp(n)}}
\leq -\frac{1}{2\lambda\sqrt{n}}\int_{x=0}^{\sqrt{9}}x p(x)dx + \frac{2}{\lambda}\sqrt{\frac{1}{n\exp(n)}}.
\]
Since $p(x)$ is the standard Gaussian distribution, it can be numerically checked that this is at most
\[
-\frac{0.19}{\lambda\sqrt{n}}+ \frac{2}{\lambda}\sqrt{\frac{1}{n\exp(n)}} =
\frac{1}{\lambda \sqrt{n}}\left(-0.19+2\sqrt{\exp(-n)}\right) \leq -\frac{1}{6\lambda\sqrt{n}}.
\]
So, we finally get $\E[\hat{w}_1] \leq -1/(6\lambda\sqrt{n})$.

Now, we show that this expected value of $\hat{w}_1$ is far away from $\wopt$. $\wopt$ is not hard to calculate: It satisfies
\[
\wopt+\exp(\wopt)=0,
\]
and it can be calculated numerically that $\wopt=-0.5671.. > -3/5$. Moreover, we assume $\lambda\leq 1/(9\sqrt{n})$, so $\lambda\sqrt{n}\leq 1/9$ and thus it can be verified that
\[
\wopt-\E[\hat{w}_1] > -\frac{3}{5}+\frac{1}{6\lambda\sqrt{n}} \geq \frac{1}{10~\lambda\sqrt{n}}.
\]
Note that this is always a positive quantity. As a result, using Jensen's inequality, we get
\[
\E[(\wopt-\wavg)^2] \geq (\wopt-\E[\wavg])^2 =  (\wopt-\E[\hat{w}_1])^2 \geq \frac{1}{100~\lambda^2 n}.
\]
Moreover, by $\lambda$-strong convexity of $f$, we have that
\[
\E[F(\wavg)-F(\wopt)] \geq \E\left[\frac{\lambda}{2}(\wavg-\wopt)^2\right] \geq \frac{1}{200~\lambda n}.
\]
Finally, it is known that by performing empirical risk minimization over all $N=nm$ instances, and using the fact that $\E[\norm{\nabla_w f(w,z)}^2]$ is bounded by a constant, we get
\[
\E[(\werm-\wopt)^2] \leq \Ocal\left(\frac{1}{\lambda^2 n m}\right)
\]
(see \cite{zhang2013communication})
and
\[
\E[F(\werm)-F(\wopt)] \leq \Ocal\left(\frac{1}{\lambda n m}\right)
\]
(see Equation \eqref{eq:Fwerm}). Combining the four inequalities above gives us the theorem statement.

\subsection{Bias Correction Also Fails}

In \cite{zhang2013communication}, which analyzes one-shot parameter averaging, the authors noticed that the analysis fails for small values of $\lambda$, and proposed a modification of the simple averaging scheme, designed to reduce bias issues. Specifically, given a parameter $r\in [0,1]$, each machine subsamples $rn$ examples without replacement from its dataset, and computes the optimum $\hat{w}_{2,k}$ with respect to this subsample. Then, it computes the optimum $\hat{w}_{k,1}$ over the entire dataset, and returns the weighted combination $\hat{w}_{k}=(\hat{w}_{k,1}-r \hat{w}_{k,2})/(1-r)$. Unfortunately, the analysis still results in lower-order terms with bad dependence on $\lambda$, and it's not difficult to extend our construction from \thmref{thm:lowbound} to show that this bias-corrected version of the algorithm still fails (at least, if $r$ is chosen in a fixed manner).

For simplicity, we will only sketch the derivation for a fixed choice of $\lambda$ given $n$, namely $\lambda = 1/10\sqrt{n}$, and for $r=1/2$. Also, we assume for simplicity that $\Wcal=\reals$ (to avoid tedious dealings with small Gaussian tails). With this choice, the returned solution becomes $\hat{w}_{k}=2\hat{w}_{k,1}-\hat{w}_{k,2}$. The distribution of $\hat{w}_{k,1}$, using the same derivation as in the proof of the theorem, is determined by
\[
\hat{w}_{k,1}+\exp(\hat{w}_{k,1}) = \frac{1}{\lambda \sqrt{n}}\tilde{z} = 10\tilde{z}
\]
where $\tilde{z}$ has a standard Gaussian distribution. As to $\hat{w}_{k,2}$, its distribution is similar to that of $\hat{w}_{k,1}$ with the same choice of $\lambda$ but only half as many points, hence
\[
\hat{w}_{k,2}+\exp(\hat{w}_{k,2}) = \frac{1}{\lambda \sqrt{n/2}}\tilde{z} = 10\sqrt{2}~\tilde{z}.
\]
By a numerical calculation, one can verify that $\E[\hat{w}_k] = 2\E[\hat{w}_{k,1}]-\E[\hat{w}_{k,2}] \approx 2*(-3.3)-(-4.8) = -1.8$.
In contrast, $w^*=-0.5671...$ as discussed in the proof. Thus, the bias is constant and does not scale down with the data size, getting a similar effect as in \thmref{thm:lowbound}

\section{Proof of Theorem \ref{thm:main}}\label{app:thmmainproof}

For any $w^{(t-1)}$, the optimal solution is always given by:
\begin{equation}\label{eq:wermquad}
\werm = \arg\min_w \phi(w) = w^{(t-1)} - H^{-1} \nabla \phi(w^{(t-1)}) .
\end{equation}
Following \eqref{eq:wiquad}, we have:
\begin{align*}
w^{(t)} &= w^{(t-1)} - \eta \left(\frac{1}{m} \sum
    (H_i+\mu I)^{-1}\right)
  \nabla \phi(w^{(t-1)}) \\
&=   w^{(t-1)} - \eta \tilde{H}^{-1} \nabla \phi(w^{(t-1)}).
\end{align*}
Therefore
\begin{equation}
  \label{eq:wtminuswerm}
  w^{(t)}-\werm = (H^{-1}-\eta \tilde{H}^{-1}) \nabla \phi(w^{(t-1)}) = (I - \eta \tilde{H}^{-1} H) (w^{(t-1)}-\werm).
\end{equation}
where for the last equality we rearranged \eqref{eq:wermquad} to
calculate $\nabla \phi(w^{(t-1)})=H(w^{(t-1)}-\werm)$.  Bounding
$\norm{Av}\leq\norm{A}_2\norm{v}$ and iterating \eqref{eq:wtminuswerm}
leads to the desired result.

\section{Proof of Lemma \ref{lem:HtildeH}}\label{app:HtildeH}
We will need two auxiliary lemmas:
\begin{lemma}\label{lem:hh1}For any positive definite matrix $H$:
\[
\norm{I-(H+\mu I)^{-1} H} = \frac{\mu}{\lambda+\mu},
\]
where $\lambda$ is the smallest eigenvalue of $H$.
\end{lemma}
\begin{proof}
Write $H=USU^\top$, then
\begin{align*}
&\norm{I-(H+\mu I)^{-1} H} = \norm{UIU^\top-(U(S+\mu I)U^\top)^{-1}USU^\top}
=  \norm{UIU^\top-U(S+\mu I)^{-1}U^\top USU^\top}\\
&= \norm{U\left(I - (S+\mu I)^{-1}S\right)U^\top}
=  \norm{I - (S+\mu I)^{-1}S}.
\end{align*}
This equals one minus the smallest element on the diagonal of the diagonal matrix $(S+\mu I)^{-1}S$, which is $\lambda/(\lambda+\mu)$.
\end{proof}

\begin{lemma}\label{lem:hh2}
Let $A$ be a positive definite matrix with minimal eigenvalue $\gamma$ which is larger than some $\mu>0$, and $\{\Delta_i\}_{i=1}^{m}$ matrices of the same size, such that $\max_i \norm{\Delta_i}\leq \beta$ and $\beta < \gamma$. Then
\[
\left\|\left(\frac{1}{m}\sum_{i=1}^{m}\left(A+\Delta_i\right)^{-1}-A^{-1}\right)(A-\mu I)\right\|\leq \frac{2\beta^2}{\gamma\left(\gamma-\beta\right)}
\]
\end{lemma}
\begin{proof}
For any $i$, we have
\[
(A+\Delta_i)^{-1} = \left(A(I+A^{-1}\Delta_i)\right)^{-1}= (I+A^{-1}\Delta_i)^{-1}A^{-1}.
\]
Note that $\norm{A^{-1}\Delta_i}\leq \norm{A^{-1}}\norm{\Delta_i} \leq \frac{1}{\gamma}\beta< 1$. Therefore, we can use the identity
\[
(I+C)^{-1}=\sum_{r=0}^{\infty}(-1)^r C^r
\]
which holds for any $C$ such that $\norm{C}< 1$. Using this with $C=A^{-1}\Delta_i$ and plugging back, we get
\[
(A+\Delta_i)^{-1} = \sum_{r=0}^{\infty}(-1)^r\left(A^{-1}\Delta_i\right)^r A^{-1} = A^{-1}-A^{-1}\Delta_i A^{-1}+\sum_{r=2}^{\infty}(-1)^r \left(A^{-1}\Delta_i\right)^r A^{-1}.
\]
Averaging over $i=1\ldots m$ and using the assumption $\sum_{i=1}^{m}\Delta_i=0$, we get
\[
\frac{1}{m}\sum_{i=1}^{m}(A+\Delta_i)^{-1}= A^{-1}+\sum_{r=2}^{\infty} \frac{(-1)^r }{m}\sum_{i=1}^{m}\left(A^{-1}\Delta_i\right)^r A^{-1}.
\]
Multiplying both sides by $(A-\mu I)$, we get
\[
\left(\frac{1}{m}\sum_{i=1}^{m}(A+\Delta_i)^{-1}\right)(A-\mu I)= A^{-1}\left(A-\mu I\right)+\sum_{r=2}^{\infty} \frac{(-1)^r }{m}\sum_{i=1}^{m}\left(A^{-1}\Delta_i\right)^r \left(I-\mu A^{-1}\right).
\]

By the triangle inequality and convexity of the norm, this implies
\begin{align*}
&\left\|\left(\frac{1}{m}\sum_{i=1}^{m}\left(A+\Delta_i\right)^{-1}-A^{-1}\right)(A-\mu I)\right\|=
\left\|\sum_{r=2}^{\infty} \frac{(-1)^r }{m}\sum_{i=1}^{m}\left(A^{-1}\Delta_i\right)^r (I-\mu A^{-1})\right\|\\
&\leq \sum_{r=2}^{\infty}\frac{1}{m}\sum_{i=1}^{m}\norm{\left(A^{-1}\Delta_i\right)^r (I-\mu A^{-1})}\\
&\leq \sum_{r=2}^{\infty}\frac{1}{m}\sum_{i=1}^{m}\norm{A^{-1}}^{r}\norm{\Delta_i}^r
\norm{I-\mu A^{-1}}\\
&\leq \sum_{r=2}^{\infty}\frac{\beta^r}{\gamma^{r}}\left(1+\frac{\mu}{\gamma}\right)
\leq \frac{2\beta^2}{\gamma^2}\sum_{r=0}^{\infty}\left(\frac{\beta}{\gamma}\right)^r = \frac{2\beta^2}{\gamma^2}\frac{1}{1-\frac{\beta}{\gamma}}
= \frac{2\beta^2}{\gamma(\gamma-\beta)}
\end{align*}
from which the result follows.
\end{proof}

We are now ready to prove Lemma \ref{lem:HtildeH}.
Using \lemref{lem:hh1}, we can upper bound $\norm{I-\tilde{H}^{-1}H}$ as
\begin{align*}
\norm{I-\frac{1}{m}\sum_{i=1}^{m}\left(H_i+\mu I\right)^{-1}H}
&\leq \norm{I-(H+\mu I)^{-1}H}+\left\|\frac{1}{m}\sum_{i=1}^{m}\left(H_i+\mu I\right)^{-1}H-(H+\mu I)^{-1}H\right\|\\
&\leq \frac{\mu}{\lambda+\mu}+\left\|\left(\frac{1}{m}\sum_{i=1}^{m}\left(H_i+\mu I\right)^{-1}-(H+\mu I)^{-1}\right)H\right\|\\
\end{align*}
Now, we use \lemref{lem:hh2} with $A = H+\mu I$ and $\Delta_i = H_i-H$ (noting that $\norm{\Delta_i}\leq \beta$), and get the bound
\[
\frac{\mu}{\lambda+\mu}+\frac{2\beta^2}{\left(\lambda+\mu\right)\left(\lambda+\mu
-\beta\right)}.
\]
assuming $\beta < \lambda+\mu$.

Now, let us assume the even stronger condition that $\beta<\frac{1}{2}(\lambda+\mu)$ (which we shall justify at the end of the proof), then we can upper bound the right hand side in the equation above by
\begin{equation}\label{eq:bbb}
\frac{\mu}{\lambda+\mu}+\frac{4\beta^2}{\left(\lambda+\mu\right)^2}.
\end{equation}
Differentiating with respect to $\mu$, we get an optimal point at
\[
\mu^{opt}=\frac{8\beta^2}{\lambda}-\lambda.
\]
If this is non-positive, it means that $\lambda^2>8\beta^2$, and moreover, that $\mu=0$, so \eqref{eq:bbb} equals $4\beta^2/\lambda^2$.
Otherwise, we pick $\mu = \frac{8\beta^2}{\lambda}-\lambda$, and \eqref{eq:bbb} becomes
\[
1-\frac{\lambda}{\frac{8\beta^2}{\lambda}}+\frac{4\beta^2}
  {\left(\frac{8\beta^2}{\lambda}\right)^{2}}
  = 1-\frac{\lambda^2}{8\beta^2}+\frac{\lambda^2}{16\beta^2}
  = 1-\frac{\lambda^2}{16\beta^2}.
\]
Combining the two cases, we get the result stated in the Lemma. Finally, it remains to justify why $\beta< \frac{1}{2}(\lambda+\mu)$. By the way we picked $\mu$, it's enough to prove that
\[
2\beta < \max\left\{\lambda,\frac{8\beta^2}{\lambda}\right\},
\]
or equivalently,
\[
2 < \max\left\{\frac{\lambda}{\beta},8\frac{\beta}{\lambda}\right\}.
\]
This is true since $\max\{x,8/x\}> 2$ for all positive $x$.

\section{Proof of \lemref{lem:hoef}}\label{app:lemhoefproof}
$H$ is the average of the Hessians of $mn$ i.i.d. quadratic functions,
all with eigenvalues at most $L$, and each $H_i$ is the average of the Hessians of $n$ i.i.d. quadratic
functions, all with eigenvalues at most $L$.  By a matrix Hoeffding's
inequality \citep{tropp12}, we have that for each $i$, with
probability $1-\delta$ over the samples received by machine $i$,
\[
\norm{H_i-\E[H_i]} \leq \sqrt{\frac{8 L^2 \log(d/\delta)}{n}}.
\]
By a union bound, we get that with probability $1-\delta$,
\[
\max_i\norm{H_i-\E[H_i]} \leq \sqrt{\frac{8 L^2 \log(dm/\delta)}{n}}.
\]
Moreover, we have $\E[H_i]=\E[H]$ and $H=\frac{1}{m}\sum_i H_i$, so if this event occurs, we also have
\[
\norm{H-\E[H]} \leq \sqrt{\frac{8 L^2 \log(d/\delta)}{n}}.
\]
Combining these, we get that with probability $1-\delta$,
\[
\max_i\norm{H_i-H} \leq \max_i\norm{H_i-\E[H]}+\norm{H-\E[H]} \leq
\sqrt{\frac{32~L^2 \log(dm/\delta)}{n}}. \qedhere
\]

\section{Proof of Theorem \ref{thm:stoch}}\label{app:thmstochproof}
Plugging \ref{lem:hoef} into Lemma \ref{lem:HtildeH}, and noting that the strong convexity of the instantaneous losses implies
$\hat{F}(w)$ is $\lambda$ strongly convex\footnote{In fact, it is
  enough to require that $\hat{F}(w)$ is $\lambda$-strongly convex,
  and it is not necessary to require strong convexity of $f(w,z)$ for
  each individual $z$.  However, requiring that the population objective
  $F(w)$ is $\lambda$-strongly convex might not be sufficient if
  $\lambda < L/\sqrt{n}$, e.g.~when $\lambda \propto 1/\sqrt{nm}$.}, we obtain
\begin{equation}\label{eq:finalHtildeH}
\norm{I-\tilde{H}^{-1}H} \leq \begin{cases}\frac{128(L/\lambda)^2\log(dm/\delta)}{n}&\text{if}~~~\frac{128(L/\lambda)^2\log(dm/\delta)}{n}\leq \frac{1}{2}\\1-\frac{n}{512 (L/\lambda)^2\log(dm/\delta)}&\text{otherwise}.\end{cases}
\end{equation}

By smoothness of $\hat{F}$, we have $\hat{F}(w^{(t)})-\hat{F}(\werm)\leq \frac{L}{2}\norm{w^{(t)}-\werm}^2$, and
therefore \thmref{thm:main} implies that
\[
 \hat{F}(w^{(t)})-\hat{F}(\werm) \leq \frac{L}{2}\norm{w^{(0)}-\werm}^2\norm{I - \eta \tilde{H}^{-1} H}^{2t}.
\]
This means that to get optimization error $\leq \epsilon$, the number of iterations required is
\begin{equation}\label{eq:linearb}
\frac{\log\left(\frac{L\norm{w^{(0)}-\werm}^2}{2\epsilon}\right)}
{-2\log\left(\norm{I - \eta \tilde{H}^{-1} H}\right)}
\end{equation}
Considering \eqref{eq:finalHtildeH}, if the first case holds, then the
denominator in \eqref{eq:linearb} is at least $2\log(2)$ and we get
that the number of iterations required is
$\Ocal\left(\log\left(\frac{L\norm{w^{(0)}-\werm}}{\epsilon}\right)\right)$.
If the second case in \eqref{eq:finalHtildeH} holds, we have
\[
\log\left(\norm{I - \eta \tilde{H}^{-1} H}\right) \leq \log\left(1-\frac{n}{512 (L/\lambda)^2\log(dm/\delta)}\right)
\leq -\frac{n}{512 (L/\lambda)^2\log(dm/\delta)},
\]
which implies that the iteration bound \eqref{eq:linearb} is at most
\[
\frac{256(L/\lambda)^2\log(dm/\delta)}{n}\log\left(\frac{L\norm{w^{(0)}-\werm}^2}{2\epsilon}\right).
\qedhere
\]

\section{Proof of \thmref{thm:general-loss}}\label{app:general-loss}
We begin with the following lemma:
  \begin{lemma}
    Under the conditions of Theorem \ref{thm:general-loss}, the
    following inequalities hold:
    \begin{equation}
      (\nabla h_i(w') - \nabla h_i(w))^\top(w'-w) \geq \frac{1}{L_i+\mu} \|\nabla h_i(w')- \nabla h_i(w)\|_2^2
      \label{eqn:dh}
  \end{equation}
  and
    \begin{equation}
    \|\nabla \phi(w)\|_2^2 \geq \lambda (\phi(w)-\phi(\werm)) . \label{eqn:dphi}
  \end{equation}
\end{lemma}
  \begin{proof}
    The smoothness of $h_i$ implies that its conjugate $h_i^*$ is $1/(L_i+\mu)$ strongly convex.
    Let $u'= \nabla h_i(w')$ and $u=\nabla h_i(w)$, then $w'= \nabla h_i^*(u')$ and $w=\nabla h_i^*(u)$.
    We have
    \[
    (\nabla h_i(w') - \nabla h_i(w))^\top(w'-w)
    =     (\nabla h_i^*(u') - \nabla h_i^*(u))^\top(u'-u) \geq \frac{1}{L_i+\mu} \|u'-u\|_2^2 .
    \]
    This proves \eqref{eqn:dh}.

Since $\nabla \phi(\werm)=0$,
    $\|\nabla \phi(w)\|_2^2  = \|\nabla \phi(w)-\nabla \phi (\werm)\|_2^2$.
    From $\|\nabla \phi(w)- \nabla \phi(\werm)\|_2 \geq \lambda \|w-\werm\|_2$, we obtain
    \begin{align*}
      \|\nabla \phi(w)- \nabla \phi(\werm)\|_2^2 \geq& \lambda \|\nabla \phi(w)-\nabla \phi(\werm)\|_2 \|w-\werm\|_2 \\
      \geq& \lambda (\nabla \phi(w)-\nabla \phi(\werm))^\top (w-\werm) \\
      =& \lambda \nabla \phi(w)^\top (w-\werm) \geq \lambda (\phi(w)-\phi(\werm) ) .
  \end{align*}
  This proves \eqref{eqn:dphi}.
\end{proof}
We are now ready to prove the Theorem.
 At iteration  $t$, we have the following first order equation:
\begin{equation}
\nabla h_i(w_i^{(t)}) - \nabla h_i(w^{(t-1)}) = - \eta \nabla \phi(w^{(t-1)}) . \label{eqn:1st-order}
\end{equation}
Therefore,
\begin{align*}
&\phi(w_i^{(t)}) \\
=& \phi(w^{(t-1)})+ \nabla \phi(w^{(t-1)})^\top (w_i^{(t)}-w^{(t-1)}) + D_\phi(w_i^{(t)};w^{(t-1)}) \\
=& \phi(w^{(t-1)}) - \frac1\eta (\nabla h_i(w_i^{(t)})- \nabla h_i(w^{(t-1)}))^\top (w_i^{(t)}-w^{(t-1)})  + D_\phi(w_i^{(t)};w^{(t-1)}) \\
\leq & \phi(w^{(t-1)}) - \frac1\eta (\nabla h_i(w_i^{(t)})- \nabla h_i(w^{(t-1)}))^\top (w_i^{(t)}-w^{(t-1)})
+ \frac{L}{2} \|w_i^{(t)}-w^{(t-1)}\|_2^2 \\
\leq & \phi(w^{(t-1)}) - \frac1\eta (\nabla h_i(w_i^{(t)})- \nabla h_i(w^{(t-1)}))^\top (w_i^{(t)}-w^{(t-1)})
+ \frac{L}{2(\lambda_i+\mu)^2}\|\nabla h_i(w_i^{(t)})-\nabla h_i(w^{(t-1)})\|_2^2 \\
\leq& \phi(w^{(t-1)}) - \left[\frac{\eta^{-1}}{\mu+L_i} - \frac{L}{2(\lambda_i+\mu)^2} \right]\|\nabla h_i(w_i^{(t)})- \nabla h_i(w^{(t-1)})\|_2^2 \\
=& \phi(w^{(t-1)}) - (\rho_i/\lambda) \|\nabla \phi(w^{(t-1)})\|_2^2 .
\end{align*}
where
$\rho_i = \left[\frac{1}{\mu+L_i} - \frac{\eta  L}{2(\mu+\lambda_i)^2}\right] \eta \lambda$.
In the above derivations, the first inequality uses the smoothness of $\phi$;
the second inequality uses the strong convexity of $h_i$;  the third inequality uses \eqref{eqn:dh}; the second and the last equalities use
\eqref{eqn:1st-order}.

Therefore
\begin{align*}
\phi(w^{(t)}) \leq& \frac1m \sum_{i=1}^m \phi(w_i^{(t)}) \\
\leq&  \frac1m \sum_{i=1}^m [\phi(w^{(t-1)}) - (\rho_i/\lambda) \|\nabla \phi(w^{(t-1)})\|_2^2 ]
= \phi(w^{(t-1)}) - (\rho/\lambda) \|\nabla \phi(w^{(t-1)})\|_2^2 \\
\leq& \phi(w^{(t-1)}) - \rho (\phi(w^{(t-1)})- \phi(\werm)) ,
\end{align*}
where the first inequality is Jensen's and the third inequality is due to \eqref{eqn:dphi}. As a result, we get
\[
\phi(w^{(t)})-\phi(\werm) \leq \phi(w^{(t-1)}) - \phi(\werm) -\rho (\phi(w^{(t-1)})- \phi(\werm))
= (1-\rho)(\phi(w^{(t-1)}) - \phi(\werm)).
\]
The desired bound follows by recursively applying the above inequality.

\section{Proof of \thmref{thm:general-loss-refined}}\label{app:general-loss-refined}

We have
\begin{align*}
\phi(\werm)=& \phi(w^{(t-1)})  + \nabla \phi(w^{(t-1)})^\top (\werm - w^{(t-1)}) + D_\phi(\werm;w^{(t-1)}) \\
=& \phi(w^{(t)}) - \nabla \phi(w^{(t-1)})^\top (w^{(t)}-w^{(t-1)}) - D_\phi(w^{(t)};w^{(t-1)})  + \nabla \phi(w^{(t-1)})^\top (\werm - w^{(t-1)}) + D_\phi(\werm;w^{(t-1)}) \\
=& \phi(w^{(t)})  + \nabla \phi(w^{(t-1)})^\top (\werm - w^{(t)}) - D_\phi(w^{(t)};w^{(t-1)}) + D_\phi(\werm;w^{(t-1)}) \\
=& \phi(w^{(t)})  + \nabla \phi(w^{(t-1)})^\top (\werm - w^{(t-1)}) - D_\phi(w^{(t)};w^{(t-1)}) + D_\phi(\werm;w^{(t-1)})
+ \nabla \phi(w^{(t-1)})^\top (w^{(t-1)}-w^{(t)}) \\
=& \phi(w^{(t)}) + \nabla \phi(w^{(t-1)})^\top (\werm - w^{(t-1)}) - D_\phi(w^{(t)};w^{(t-1)}) + D_\phi(\werm;w^{(t-1)}) \\
& + \eta^{-1} [D_h(w^{(t-1)};w^{(t)}) + D_h(w^{(t)};w^{(t-1)})] \\
\geq & \phi(w^{(t)}) + \nabla \phi(w^{(t-1)})^\top (\werm - w^{(t-1)}) + D_\phi(\werm;w^{(t-1)}) + \eta^{-1} D_h(w^{(t-1)};w^{(t)}) ,
\end{align*}
where in the last inequality, we have used the assumption that $D_\phi(w^{(t)};w^{(t-1)}) \leq \eta^{-1} D_h(w^{(t)};w^{(t-1)})$.
This implies that
\begin{equation}
 D_h(w^{(t-1)};w^{(t)}) + \eta \nabla \phi(w^{(t-1)})^\top (\werm - w^{(t-1)})
\leq \eta [\phi(\werm) - \phi(w^{(t)}) - D_\phi(\werm;w^{(t-1)}) ]
. \label{eqn:general-loss-refined-proof1}
\end{equation}
Therefore we have
\begin{align*}
& D_h(\werm;w^{(t)}) - D_h(\werm;w^{(t-1)}) \\
=& D_h(w^{(t-1)};w^{(t)}) + (\nabla h(w^{(t-1)})-\nabla h(w^{(t)}))^\top (\werm-w^{(t-1)}) \\
=& D_h(w^{(t-1)};w^{(t)}) + \eta \nabla \phi(w^{(t-1)})^\top (\werm-w^{(t-1)}) \\
\leq & \eta [\phi(\werm) - \phi(w^{(t)}) - D_\phi(\werm;w^{(t-1)}) ] \\
\leq & - \eta D_\phi(\werm;w^{(t-1)}) \\
\leq & - \eta \gamma D_h(\werm;w^{(t-1)}) ,
\end{align*}
where the first inequality is due to \eqref{eqn:general-loss-refined-proof1}, the second inequality
comes from the inequality $\phi(\werm) \leq \phi(w^{(t)})$,
and the third inequality uses the assumption that $D_\phi(\werm;w^{(t-1)}) \geq \gamma D_h(\werm;w^{(t-1)})$.
We thus obtain
\[
D_h(\werm;w^{(t)}) \leq (1-\eta \gamma) D_h(\werm;w^{(t-1)}) ,
\]
and this implies the desired result.

\end{document}